\documentclass[10pt,dvipsnames]{article} 

\PassOptionsToPackage{numbers, sort&compress}{natbib}
\usepackage[preprint]{tmlr}

\usepackage[colorlinks=true,citecolor=BlueViolet,linkcolor=Blue]{hyperref} 
\usepackage{mpemath}
\RequirePackage[capitalise,noabbrev]{cleveref}
\usepackage{graphicx} 

\SetKwInOut{Input}{Input}
\SetKwInOut{Output}{Output}

\usepackage{booktabs}

\newcommand{\Nats}{\mathbb{N}}

\DeclareMathOperator{\kl}{KL}
\DeclareMathOperator{\ER}{DR}
\DeclareMathOperator{\ess}{ess}

\DeclareMathOperator{\qvar}{QuickVaR}
\DeclareMathOperator{\ews}{EWS}

\DeclareMathOperator*{\argmin}{arg\,min}

\newcommand{\quant}{{\mathfrak{q}}}

\renewcommand{\cite}{\citep}

\IncMargin{2em}

\title{Computing Monetary Risk Measures in Linear Time} 
\author{Palash Agrawal\thanks{Equal contribution.}, Gersi Doko\footnotemark[1], Maeve Burwell, Marek Petrik\footnotemark[1]}
\date{}

\begin{document}

\maketitle

\begin{abstract}
Monetary risk measures have gained popularity for expressing decision-makers' risk aversion. Value-at-Risk (VaR) and Conditional-Value-at-Risk (CVaR), in particular, are used commonly for this purpose. This paper proposes new efficient algorithms to compute these risk measures for a discrete random variable in expected linear time with respect to the size of its domain. First, we propose a QuickVaR algorithm that computes the VaR of a discrete random variable. Then, we leverage QuickVaR to propose QuickDivergence, an algorithm for computing a class of $\varphi$-divergence risk measures, including the popular CVaR risk measure. The QuickVaR algorithm adapts the well-known Quickselect algorithm, while QuickDivergence builds on polymatroid optimization algorithms. Numerical results show that our new algorithms offer an order-of-magnitude speedup for large domains, and a library implementation of the algorithms is available at \url{https://github.com/RiskAverseRL/RiskMeasures.jl}.     
\end{abstract}

\section{Introduction}\label{sec:intro}

Monetary measures of risk generalize the expectation operator to be able to represent the risk aversion of decision-makers in domains that range from robotics~\cite{akella2024risk,majumdar2020should,benrabah2024review} to finance~\cite{Follmer2016} to infrastructure maintenance~\cite{Inzunza2016} to disaster response~\cite{Barahona2013}. Value-at-Risk~(VaR) and Conditional-Value-at-Risk~(CVaR) are the most popular risk measures and have made inroads in reinforcement learning, machine learning, and decision making~\cite{Petrik2024}. Machine learning and reinforcement learning algorithms typically must compute the risk measures quickly for large discrete random variables. For example, in risk-averse reinforcement learning, one must evaluate the risk measure in every iteration of the algorithm, which can represent a significant computational bottleneck~\cite{Ho2021,Ho2018,Ho2022,Hau2023a,Hau2025}. 

In this paper, we propose linear-time algorithms for computing risk measures of discrete random variables. First, we develop \emph{QuickVaR}, an algorithm that computes VaR in linear time and generalizes the well-known quickselect algorithm. Second, we develop \emph{QuickDivergence}, an algorithm that computes, in linear time, $\varphi$-divergence risk measures~\cite{Ahmadi-Javid2012} that satisfy a specific piecewise linearity property. This algorithm reduces $\varphi$-divergence risk measures to linear minimization problems over polymatroids~\cite{Follmer2016}. We illustrate this algorithm for the well-known CVaR risk measure and a related TVaR risk measure. Both of these risk measures are in the $\varphi$-divergence family and satisfy the required linearity property.  

Our algorithms' time complexity improves on the state of the art by a logarithmic factor. Virtually all implementations of VaR and CVaR algorithms require sorting the input arrays. Unbounded sorting requires $\Omega(n \log n)$ operations where $n$ is the size of the probability space~\cite{Follmer2016}. Our numerical results show that the logarithmic acceleration can lead to an order-of-magnitude speedup for realistic large problems without incurring any penalty for small problems. In addition, our linear-time algorithms have simple implementations that add minimal complexity over standard algorithms.

Linear-time algorithms similar to our QuickVaR and QuickDivergence have been studied in many domains, and we summarize some of the most relevant. Some examples of such structured optimization problems are the projection onto the simplex~\cite{adam2019projectionscanonicalsimplexadditional, Wang2013Simplex}, projection onto the $k$-capped simplex~\cite{kcapped_ang,Lim2016}, linear optimization on the intersection~of $L_p$ or $\varphi$-divergence balls and the probability simplex in robust optimization~\cite{Ho2022}, optimization for $\varphi$-divergence risk measures~\cite{Ahmadi-Javid2012}, and on polymatroids~\cite{Korte2012}. While some of these papers discuss the possibility of designing linear-time algorithms, they usually do not study them, with a few exceptions~\cite{Condat2016,Lim2016}. Finally, QuickVaR is closely related to the computation of the weighted quantile but returns the same value only when the quantile is unique~\cite[chapter~9]{cormen2001introduction}. 

The remainder of the paper is organized as follows. \Cref{sec:preliminaries} describes the necessary background including risk measures, $\varphi$-divergence, and polymatroid optimization. Then, \Cref{sec:var-linear} proposes \emph{QuickVaR} to compute VaR in linear time. \Cref{sec:phi-divergence} builds on QuickVaR to develop \emph{QuickDivergence} that computes certain $\varphi$-divergence risk measures in linear time.  We evaluate the algorithms numerically on synthetic and realistic problems in \cref{sec:experiments}.

\section{Preliminaries}\label{sec:preliminaries}

\emph{Notation}: We restrict our attention to random variables defined on finite probability spaces $(\Omega, 2^{\Omega}, \bm{p})$ where $\Omega := \left\{ 1, \dots , n \right\}$, $\bm{p} \in \Delta_n$, and $\Delta_n$ is the $n$-dimensional probability simplex. We use a tilde to indicate a random variable, such as $\tilde{x}\colon \Omega \to \Real$. Vectors are denoted with a lower-case bold font, such as $\bm{x} \in \Real^n$ and indexed as $\bm{x}_i \text{ for } i = 1, \dots , n$. We use the function notation $\tilde{x}$ and vector notation $\bm{x}$ interchangeably, such that $x_\omega = x(\omega), \forall \,  \omega \in \Omega$. Sets are denoted with calligraphic letters, $\bar{\Real} := \Real \cup \left\{ -\infty, \infty \right\}$ denotes the extended real line, and $\Real_{\ge 0}$ are non-negative reals. The symbol $\Nats$ denotes all 0-based natural numbers and $a{:}b := a, a+1, \dots , b$ for $a \le b \in \mathbb{N}$ and the empty set otherwise. We use $\bm{x}_{a{:}b}$ to denote a subset of elements of $\bm{x}$ and comparisons such as $\bm{x}_{a{:}b} < y$ are interpreted element-wise. Finally, the symbol $\bm{1} \in \Real^n$ denotes a vector of all ones and $\bm{1}_{\mathcal{O}} \in \Real^n$ for $\mathcal{O} \subseteq \Omega$ is an indicator vector of the set $\mathcal{O}$.

\emph{Monetary risk measures} $\rho\colon \Delta_n \times \Real^n \to \bar{\Real}$ are functions of a probability measure and a random variable that are monotone and translation invariant~\cite{Follmer2016}. Risk measures generalize the expectation operator to account for uncertainty in a random variable's distribution. We apply risk measures to random variables representing rewards and suppose that higher risk values are preferable. For the sake of streamlining the notation, we omit the probability measure when it is obvious from the context and only write $\rho[\tilde{x}]$ to compute the risk value.

\emph{Value-at-Risk}~(VaR) is a popular risk measure and is defined for $\alpha\in [0,1)$ and $\tilde{x}\in \mathbb{R}^n$ as~\cite{Follmer2016,Shapiro2014}
\begin{equation}
  \label{eq:var-definition}
  \var{\alpha}{ \tilde{x} }
  \;:=\;
    \max\; \left\{\tau \in \bar{\Real} \mid \P{\tilde{x} < \tau}  \le \alpha\right\} =
    \max\; \left\{\tau \in \bar{\Real} \mid \P{\tilde{x} \ge  \tau}  \ge 1 - \alpha\right\},
\end{equation}
where the maximum exists because $\tau \mapsto \P{\tilde{x} < \tau}$ is lower semi-continuous. One can readily see that $\varo_0[\tilde{x}] = \ess \inf [\tilde{x}]$, and typically the risk measure is extended as $\varo_{1}[ \tilde{x} ] = \infty$.

\begin{table}
    \caption{Parameters for relevant $\varphi$-divergence risk measures for $\alpha\in (0,1)$.} \label{tab:phi-divergence-risk}
    \centering
    \begin{tabular}{l|c|c|c}
         \toprule
         Risk measure & $\varphi(y)$  & $\varphi\opt(x)$ & $\beta$ \\
         \midrule        
         $\mathbb{E}$           & $\begin{cases} 0 &\text{if } y = 1 \\ \infty & \text{otherwise}  \end{cases}$ & $x$ & $0$  \\
      \midrule
      $\cvaro_{\alpha}$     & $\begin{cases} 0 &\text{if } y \in [0, \frac{1}{\alpha}] \\ \infty & \text{otherwise}  \end{cases}$ & $\frac{1}{\alpha} \cdot \max \left\{ 0,  x \right\} $& $0$ \\
      \midrule
      $\evaro_{\alpha}$     & $\begin{cases} y \log y &\text{if } y > 0 \\ 0 &\text{if } y =  0 \\  \infty &\text{if } y < 0 \end{cases}$ & $\exp(x - 1)$ & $\log \frac{1}{\alpha}$ \\
      \midrule
         $\tvaro_{\alpha}$     & $|y - 1|$ & $ \begin{cases} x &\text{if } |x| \le 1\\ \infty &\text{otherwise} \end{cases} $& $\min \left\{  \sqrt{ 2 \log \frac{1}{\alpha} }, 2\right\}$ \\
         \bottomrule
    \end{tabular}
\end{table}

A \emph{$\varphi$-divergence risk measure} $\ER_{\varphi, \beta }\colon \mathbb{R}^n \to  \bar{\Real}$ is defined for $\varphi\colon \bar{\Real} \to \bar{\Real}$ and $\beta\in \Real_{\ge 0}$ as~\cite{Ahmadi-Javid2012}
\begin{align} 
  \ER_{\varphi, \beta}[ \tilde{x} ]
\label{eq:fdiv-primal}
  &:= \inf \left\{\bm{q}\tr \bm{x} \mid \bm{q}\in \Delta_n, H_{\varphi}(\bm{q} \| \bm{p}) \le \beta,
    \bm{q} \ll \bm{p}
    \right\}
    \\
  \label{eq:fdiv-dual}
  &= \sup_{z\in \Real} \left(  z - \inf_{t>0} t\cdot \E { \varphi\opt \left( \frac{z - \tilde{x}}{t} + \beta \right)} \right),
\end{align}
where $\varphi\opt(x)$ is the convex conjugate of $\varphi(y)$, $\varphi$ is a closed, convex function that satisfies $\varphi(1) = 0$, and $H_{\varphi}$ is a $\varphi$-divergence measure~\cite{Bayraksan2015}:
\[
  H_\varphi(\bm{q} \| \bm{p}) := \sum_{\omega \in \Omega} p_{\omega} \cdot \varphi\left(\frac{q_{\omega}}{p_{\omega}}\right),
\]
defined for $\bm{q} \ll \bm{p}$. Here, $\bm{q} \ll \bm{p}$ denotes that $\bm{q}$ is absolutely continuous with respect to $\bm{p}$, meaning that $p_{\omega} = 0 \implies q_{\omega} = 0$ for each $\omega\in 1{:}n$. We also adopt the common convention that $1 / 0 = \infty$ and $0 \cdot \infty = 0$~\cite{rockafellar1996variational}. It is important to note that $\varphi$-divergence risk measures are coherent, which means they are concave and positively homogeneous. Any coherent risk measure $\rho\colon \mathbb{R}^n\to \bar{\Real}$ satisfies that $\rho[\tilde{x}] = \inf_{\bm{q}\in\mathcal{Q}} \bm{q}\tr\bm{x}$ for some closed convex non-empty set $\mathcal{Q} \subseteq \Delta_n$ \citep[proposition~2.84]{Follmer2016}. 

\Cref{tab:phi-divergence-risk} summarizes several common $\varphi$-divergence risk measures; see \cref{sec:cvar-tvar-are} for the relevant definitions for $\alpha\in [0,1]$. We focus on the \emph{Conditional-Value-at-Risk}~(CVaR), which is a $\varphi$-divergence risk measure and defined equivalently for $\tilde{x} \in \mathbb{R}^n$ and $\alpha \in (0,1]$ as~\cite{Follmer2016}
\begin{equation} \label{eq:primal_cvar}
    \cvaro_\alpha[ \tilde{x} ]
    \; :=\; 
    \max_{z \in \Real} \left(z - \frac{1}{\alpha} \cdot \E{z -\tilde{x}}_+\right)
    = 
\min \left\{ \bm{x}\tr \bm{q} \mid \bm{q} \in \Delta_{n}, \, \bm{q} \le \alpha^{-1} \bm{p}
\right\},
\end{equation}
and extended continuously to $\cvaro_0[\tilde{x}] = \ess\inf[\tilde{x}]$.

We also study another $\varphi$-divergence risk measure, which we call \emph{Total-Value-at-Risk}~(TVaR) and which is defined for $\alpha\in (0,1]$ as
\begin{equation} \label{eq:tvar-dual}
  \tvaro_{\alpha}[\tilde{x}]
  \;:=\; 
\min \left\{ \bm{x}\tr \bm{q} \mid \bm{q} \in \Delta_{n}, \, \|\bm{p} - \bm{q}\|_1 \le \min \left\{  \sqrt{2 \log \frac{1}{\alpha }}, 2\right\}, \bm{q} \ll \bm{p}
\right\},
\end{equation}
and extended continuously to $\tvaro_{0}[\tilde{x}] = \ess\inf [\tilde{x}]$.

We also refer to the \emph{Entropic-Value-at-Risk}~(EVaR), which is the $\varphi$-divergence risk measure obtained from $\varphi(y) := y \log y$ for $y > 0$ and $\beta = \log \frac{1}{\alpha }$ for $\alpha > 0$ in \cref{tab:phi-divergence-risk}. Because the corresponding $\varphi$-divergence is the Kullback--Leibler divergence $H_{\varphi}(\bm{q} \| \bm{p}) = \kl(\bm{q} \| \bm{p}) = \sum_{\omega \in \Omega} q_{\omega} \log \frac{q_{\omega}}{p_{\omega}}$, the EVaR is equivalently defined for $\alpha\in (0,1]$ and $\tilde{x} \in \mathbb{R}^n$ as~\cite{Ahmadi-Javid2012}
\begin{equation} \label{eq:evar-dual}
  \evaro_{\alpha}[\tilde{x}]
  \;:=\;
  \min \left\{ \bm{x}\tr \bm{q} \mid \bm{q} \in \Delta_{n}, \, \kl(\bm{q} \| \bm{p}) \le \log \frac{1}{\alpha },
    \bm{q} \ll \bm{p}
  \right\},
\end{equation}
and extended continuously to $\evaro_{0}[\tilde{x}] = \ess\inf [\tilde{x}]$. 
  
\begin{figure}
  \centering
  \includegraphics[width=0.5\linewidth]{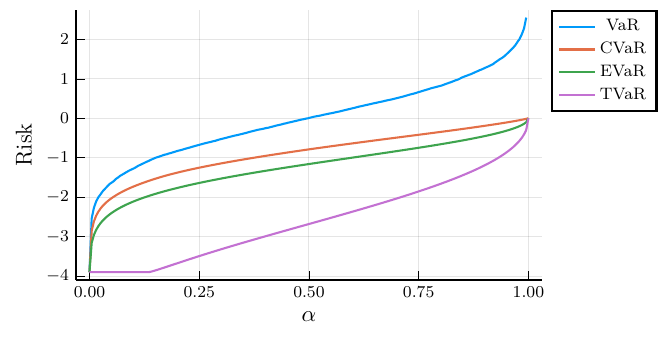}
  \caption{Comparison of risk measure values for a normal distribution with mean $0$ and standard deviation $1$ as a function of the risk level $\alpha$.}
  \label{fig:comparison}
\end{figure}
An important property of TVaR is that it bounds EVaR~\cite{Ahmadi-Javid2012,Hau2023a,Su2025} from below, as we show next and illustrate in \cref{fig:comparison}.
\begin{proposition}
For each $\tilde{x} \in \mathbb{R}^n$, the risk measures above extended to $\alpha\in [0,1]$ satisfy that
\[
  \var{\alpha}{\tilde{x}}
  \; \ge\; 
  \cvar{\alpha}{\tilde{x}}
  \; \ge\; 
  \evar{\alpha}{\tilde{x}}
  \; \ge\; 
  \tvar{\alpha}{\tilde{x}}.
\]
\end{proposition}
\begin{proof}
 The inequality for VaR and CVaR is well known, and the inequality for CVaR and EVaR follows from \citep[proposition~3.2]{Ahmadi-Javid2012}. Recall that Pinsker's inequality~\citep[theorem~4.19]{boucheron2013concentration} shows for $\bm{p}, \bm{q}\in \Delta_n$ that $\| \bm{p} - \bm{q}\|_1 \le \sqrt{2 \kl(\bm{q} \| \bm{p})}$. Then we prove that TVaR lower-bounds EVaR as
\begin{align*}
\evaro_{\alpha}[\tilde{x}]
&=
\inf \left\{ \bm{q}\tr \bm{x} \mid \bm{q}\in \Delta_n, \kl(\bm{q} \| \bm{p}) \le \log \frac{1}{\alpha }, \bm{q} \ll \bm{p} \right\} \\
&\ge
\inf \left\{ \bm{q}\tr \bm{x} \mid \bm{q}\in \Delta_n, \| \bm{p} - \bm{q} \|_1 \le \sqrt{ 2 \log \frac{1}{\alpha }} , \bm{q} \ll \bm{p} \right\} \\
&=\inf \left\{ \bm{q}\tr \bm{x} \mid \bm{q}\in \Delta_n, \| \bm{p} - \bm{q} \|_1 \le \min \left\{  \sqrt{ 2 \log \frac{1}{\alpha }}, 2\right\}, \bm{q} \ll \bm{p}\right\}
= \tvaro_{\alpha }[\tilde{x}].
\end{align*}
We use the fact that $\| \bm{p} - \bm{q}\|_1 \le 2$.
\end{proof}

\section{Computing VaR in Linear Time} \label{sec:var-linear}

We now describe the QuickVaR algorithm, show that it computes VaR, and show that it runs in linear expected time. 

The standard VaR algorithm for a discrete random variable $\bm{x}\in \Real^n$ with distribution $\bm{p}\in \Delta_n$ and risk level $\alpha$ proceeds as follows~(see listing~6.25~\cite{privault2026notes}).  It sorts the entries of $\bm{x}$ to get a permutation $\sigma \colon 1{:}n \to 1{:}n$ such that $i \mapsto x_{\sigma(i)}$ is monotone. Then, it finds an index $i \in 1{:}n$ such that 
\[
  \sum_{j = 1}^{i-1} p_{\sigma(j)} 
  \; <\;
  \alpha
  \; \le\;
  \sum_{j = 1}^i p_{\sigma(j)}.
\]
The value $x_{\sigma(i)}$ is the optimal solution in~\eqref{eq:var-definition} by \cref{thm:quantile-in-finite-omega} below. The computational complexity of this algorithm is $\Omega(n \log  n)$ because its runtime is dominated by the need to sort the entries of $\bm{x}$. For the sake of completeness, \cref{alg:slow-quantile} in the appendix summarizes the standard greedy approach to computing VaR. 

\emph{QuickVaR}, described in \cref{alg:quick-quantile}, computes VaR in $O(n)$ time by avoiding the need to sort the elements of the input random variable $\bm{x}$. In particular, QuickVaR partially sorts the indices of $\bm{x}$, analogous to the quickselect algorithm~\cite{horowitz1998computer}. 

\Cref{alg:quick-quantile} is recursive and applies to a random variable $\bm{x}$ with a distribution $\bm{p}\in \Delta_n$ as follows. In addition to the random variable, the input is $f$ (front) and $b$ (back), which delimit the indices of $\bm{x}$ it manipulates.  The algorithm first picks a pivot element $v$ uniformly at random from $f{:}b$. The random pivot selection is essential to achieve an average linear time complexity. Given a pivot, we use $\operatorname{partition}$ to partially sort the entries of $\bm{x}$ into three groups: 1. less than $v$: $\bm{x}_{f{:}(l-1)} < v$, 2. equal to $v$: $\bm{x}_{l{:}g} = v$, and 3. greater than $v$: $\bm{x}_{(g+1){:}b} > v$. Then, the algorithm checks for the conditions of \cref{thm:quantile-in-finite-omega} to determine which of the three groups contains the VaR value. Finally, if the VaR value is in group 1 or 3, the algorithm calls itself recursively with adjusted $f$, $b$, and $\alpha$.

The $\operatorname{partition}$ procedure in \cref{alg:quick-quantile} uses the Dutch national flag algorithm to partition the input array into the three groups. The algorithm runs in linear time, and the properties of its output are given in \cref{lem:hoare-partition} below. For the sake of completeness, we provide the reader with the algorithm in \cref{alg:partition} in \cref{sec:stand-var-algor}. 

We present \cref{alg:quick-quantile} recursively to simplify its analysis. However, a loop-based, non-recursive implementation, such as \cref{alg:quickvar-nonrecursive} in \cref{sec:stand-var-algor}, typically performs better in imperative languages. 

\begin{algorithm}
  \caption{\texttt{QuickVaR}: Recursive VaR algorithm} \label{alg:quick-quantile}
  \Input{$\bm{x} \in \Real^n$, $\bm{p} \in \Delta_n$, $\alpha \in [0,1)$, $f, b \in 1{:}n$ such that $f \le  b$ and $\sum_{j=f}^b p_j > \alpha$}
  \Output{$\varo_{\alpha}[\tilde{x}]$}
  $v \gets x_{\operatorname{rand}(f{:}b)}$ \tcp*{Pivot}
  $l, g, \bm{x}', \bm{p}' \gets \operatorname{partition}(\bm{x}, \bm{p}, v, f, b)$
  \tcp*{~\cref{alg:partition}}
  \label{ln:t}
  $t \gets \sum_{j=f}^{l-1} p'_j$ \tcp*{$t = \P{\tilde{x} < v} - \sum_{i=1}^{f-1} p_i $}
  \label{ln:e}
  $e \gets \sum_{j=l}^{g} p'_j$  \tcp*{$e = \P{\tilde{x} = v}$}
  \label{ln:firstif}
  \If(\tcp*[f]{$\alpha + \sum_{i=1}^{f-1} p_i  < \P{\tilde{x} < v}$}){$\alpha < t$}{ 
    \KwRet{$\qvar(\bm{x}', \bm{p}', \alpha, f, l-1)$}
  }
  \label{ln:secondif}
  \ElseIf(\tcp*[f]{$\P{\tilde{x} \le v} \le \alpha + \sum_{i=1}^{f-1} p_i $}){$t + e \le \alpha$}{
    \KwRet{$\qvar(\bm{x}', \bm{p}', \alpha - t - e, g+1, b)$}
  }
  \label{ln:thirdif}
  \lElse{
    \KwRet{$v$}
  }
\end{algorithm}

The following theorem states the correctness of QuickVaR.
\begin{theorem} \label{thm:var-correct}
\Cref{alg:quick-quantile} executed with $f = 1$ and $b = n$ and $\alpha\in [0,1)$ returns $\varo_{\alpha}[ \tilde{x} ]$.
\end{theorem}

Before proving \cref{thm:var-correct}, we need to show several important lemmas. The following lemma, which holds only for discrete random variables, shows that the VaR value is the probability atom that satisfies specific probability properties. 
\begin{lemma}\label{thm:quantile-in-finite-omega}
Assuming $\alpha\in [0,1)$, there exists $\omega\in\Omega$ with $v := \var{\alpha }{\tilde{x}} = x_{\omega}$ and
\begin{equation} \label{eq:var-property-prob}
  \P{\tilde{x} < v} \le \alpha < \P{\tilde{x} \le v}. 
\end{equation}
\end{lemma}
\begin{proof}
First, we show that any $v$ that satisfies~\eqref{eq:var-property-prob} solves~\eqref{eq:var-definition} optimally. The implication $\P{\tilde{x} < v} \le \alpha \implies v \in \left\{ t \in \Real \mid  \P{ \tilde{x} < t } \le \alpha \right\}$ follows immediately. Second, we prove $\alpha < \P{ \tilde{x} \le v } \implies \forall t \in \left\{ t \in \Real \mid  \P{ \tilde{x} < t } \le \alpha \right\}, v \ge t$ by contradiction. Suppose that $\exists t \in \left\{ t \in \Real \mid  \P{ \tilde{x} < t } \le \alpha \right\}, v < t$. Then, by the monotonicity of the probability operator, we get a contradiction as
 \[
   \alpha
   < \P {  \tilde{x} \le v  }
   \le \P {  \tilde{x} < t  }
   \le \alpha.
 \]
Then, the existence of $\omega$ follows from \cref{lem:prob-lumpy} because~\eqref{eq:var-property-prob} implies that 
  \[
    \P{\tilde{x} =
      \var{\alpha }{\tilde{x}}} = \P{\tilde{x} \le \var{\alpha }{\tilde{x}}}  - \P{\tilde{x} < \var{\alpha }{\tilde{x}}} > \alpha - \alpha = 0.
  \]
\end{proof}

The following lemma shows the recursive property of QuickVaR and is the main building block of \cref{thm:var-correct}.
\begin{lemma} \label{lem:quickvar-correct}
Suppose that the input conditions of \cref{alg:quick-quantile} hold and $\bm{x} \in \Real^n$ satisfies that
 \begin{equation} \label{eq:qvar-precondition}
  \mathbf{x}_{1{:}(f-1)} < x_f, \qquad
  x_{f-1} < \bm{x}_{f{:}b} < x_{b+1},  \qquad
  x_b < \bm{x}_{(b+1){:}n},
\end{equation} 
where $x_0 = -\infty$ and $x_{n+1} = +\infty$. Then \cref{alg:quick-quantile} returns $x_{\omega}, \omega\in f{:}b$ such that
\begin{equation} \label{eq:qvar-postcondition}
  \P{\tilde{x} < x_{\omega}}
  \quad \le \quad
  \alpha + \sum_{i=1}^{f-1} p_i 
  \quad < \quad \P{\tilde{x} \le x_{\omega}}.
\end{equation}
\end{lemma}
\begin{proof}
From~\eqref{eq:qvar-precondition} and \cref{lem:hoare-partition}, $t$ and $e$ on lines~\ref{ln:t} and \ref{ln:e} satisfy that
\begin{equation} \label{eq:te-equivalence}
  t = \P{\tilde{x} < v} - \sum_{i=1}^{f-1} p_i , \qquad e = \P{\tilde{x} = v}.
\end{equation}
We proceed by induction on $k = b - f$. \emph{Base case}: $k = 0$. Since $b = f$, we have that $v = x_b$ and $l = g = b$, $t = 0$, $e = p_b$, and $\alpha  \ge t$. The input condition implies that $t + e = p_b > \alpha$. Because $t \le \alpha < p_b$, the algorithm returns $v$ on line~\ref{ln:thirdif} that satisfies
\[
 \P{\tilde{x} < v} \le \alpha + \sum_{i=1}^{f-1} p_i < \P{\tilde{x} \le v}. 
\]
Above, the first inequality follows from $v = x_f$ and $0 \le \alpha$. The second inequality follows from the algorithm's input condition.  

\emph{Inductive step}. Suppose that $k>0$ and that the result holds for each $k' < k$. Then:

(i) If $\alpha < t$, the algorithm executes the first if statement on line~\ref{ln:firstif}. Because $0 \le \alpha  < t$, the recursive call satisfies the required input conditions: $f \le l-1$ and $\sum_{j=f}^{l-1} p_j = t > \alpha$. The recursive call also satisfies~\eqref{eq:qvar-precondition} by \cref{lem:hoare-partition}, and we can conclude that~\eqref{eq:qvar-postcondition} follows from the inductive hypothesis because $l-1 - f \le b-1 -f \le k-1 < k$. 

(ii) If $t + e \le \alpha$, the algorithm executes the second if statement on line~\ref{ln:secondif}. We show that the input conditions for the recursive call are met. Because $t + e = \sum_{j = f}^g p_j \le \alpha < \sum_{j=f}^b p_j$  (the second inequality is the input condition), it must be that $g < b$ and therefore $g+1 \le b$. Also, by the algorithm's input condition and algebraic manipulation, we get that
\[
 \sum_{j=g+1}^b p_j = \sum_{j=f}^b p_j - t - e > \alpha - t - e,
\]
and the input conditions hold.  Because~\eqref{eq:qvar-precondition} follows from \cref{lem:hoare-partition} and $b - g - 1 < b-f \le  k$, the inductive hypothesis implies that there exists $x_{\omega}$ such that
\[
  \P{\tilde{x} < x_{\omega}}
  \quad \le \quad
  (\alpha - t - e) + \P{\tilde{x} \le v}
  \quad < \quad \P{\tilde{x} \le x_{\omega}}. 
\]
Then,~\eqref{eq:qvar-postcondition} follows by algebraic manipulation after substituting $t,e$ from~\eqref{eq:te-equivalence}.

(iii) If $t \le \alpha < t + e$, the algorithm executes the else statement on line~\ref{ln:thirdif}, and the inequality~\eqref{eq:qvar-postcondition} holds by substituting $t,e$ from~\eqref{eq:te-equivalence}.
\end{proof}

We are now ready to prove the correctness of QuickVaR.
\begin{proof}[Proof of \cref{thm:var-correct}]
With $f = 1$ and $b = n$, the condition in~\eqref{eq:qvar-precondition} holds vacuously, and the input condition $\sum_{j=1}^n p_j = 1 > \alpha$ holds since $\alpha\in [0,1)$. Therefore, \cref{lem:quickvar-correct} shows that QuickVaR returns $x_{\omega}$ that satisfies
  \[
  \P{\tilde{x} < x_{\omega}}
  \quad \le \quad
  \alpha 
  \quad < \quad \P{\tilde{x} \le x_{\omega}}, 
\]
which satisfies~\eqref{eq:var-property-prob} and, therefore, it equals the VaR value by 
  \cref{thm:quantile-in-finite-omega}. 
\end{proof}

With the correctness of \cref{alg:quick-quantile} established in \cref{lem:quickvar-correct}, we now show its computational complexity. 
\begin{theorem} \label{prop:time-complexity}
\Cref{alg:quick-quantile} runs in expected $O(n)$ time.
\end{theorem}
\begin{proof}
The proof is analogous to the standard quickselect algorithm, as in \citep[theorem~3.3]{horowitz1998computer}, except that we need to keep track of the risk level $\alpha$. Fix some $n \in \mathbb{N}$, $\bm{x}$, and $\bm{p}$, and we study the algorithm's properties. Let $T\colon 1{:}n \to [0,\infty]$ be the \emph{expected} runtime of \cref{alg:quick-quantile} such that $T(k)$ is the expected runtime when executed with $k := b - f + 1$. To simplify the notation, we assume without loss of generality in the remainder of the proof that $f = 1$ and $b = k$. Then, let
\[
  l := \left| \left\{ i \in 1{:}k \mid  x_i < \varo_{\alpha}[\tilde{x}]\right\}   \right|.
\]
Note that by the assumption $\sum_{j = f}^b p_j > \alpha$, we have that $k \ge 1$ and $\varo_{\alpha}[\tilde{x}] \le \max_{i\in f{:}b} x_i$. Therefore, $l \in 0{:}(k-1)$. The algorithm terminates in constant time when $k = 1$, and we upper-bound the runtime for the worst-case choice of $l$ for $k > 1$ as
\begin{align}
\nonumber
  T(k)
  &\le d\cdot k + \max_{l\in 0{:}(k-1)} \frac{1}{k} \sum_{i=0}^{k-1} \left(   T(i) \cdot 1_{l < i} + T(k - i - 1) \cdot 1_{l > i}  \right) \\
\nonumber
  &= d\cdot k + \max_{l\in 0{:}(k-1)} \frac{1}{k}  \left(  \sum_{i=l+1}^{k-1} T(i)  + \sum_{i=0}^{l-1} T(k - i - 1)   \right) \\
\label{eq:quickvar-complexity}
  &= d\cdot k + \max_{l\in 0{:}(k-1)} \frac{1}{k}  \left(  \sum_{i=l+1}^{k-1} T(i)  + \sum_{i=k-l}^{k-1} T(i)   \right),
\end{align}
where $d\cdot k$ is the local runtime of the algorithm for some $d > 0$ and $T(i)$ and $T(k - i - 1)$ are the average runtimes of the recursive calls. The average over $i$ corresponds to the uniformly random selection of the pivot with $i = \left| \left\{ j \in 1{:}k \mid x_j < v \right\} \right|$ in the first line of the inequality. The last two equalities follow by algebraic manipulation. Then, an argument identical to theorem~3.3~\cite{horowitz1998computer}, since~\eqref{eq:quickvar-complexity} is the same as equation~(3.8) in~\citet{horowitz1998computer}, shows that \(T(k)\le 4\cdot d\cdot k\) for all \(k\ge 1\).
\end{proof}
 
\begin{remark}
  An attentive reader may notice that QuickVaR resembles the algorithm for computing the weighted median~\cite[section~17]{Korte2012}. However, the weighted median algorithm does not, in general, compute VaR. Define the \emph{quantile} $\quant_{\alpha},\alpha \in (0,1)$ for $\tilde{x}$ as the closed set
\[
 \quant_{\alpha}[\tilde{x}] := \left\{ z\in \Real \mid  \P{ \tilde{x} \le z } \ge \alpha, \P{ \tilde{x} \ge z } \ge 1-\alpha  \right\},
\]
with bounds $\quant^+_{\alpha}[\tilde{x}] := \max \quant_{\alpha}[\tilde{x}]$ and $\quant^-_{\alpha}[\tilde{x}] := \min \quant_{\alpha}[\tilde{x}]$. One can readily show that $\varo_{\alpha }[\tilde{x}] = \quant^+_{\alpha}[\tilde{x}]$. Although the precise definition of a weighted median varies, it is typically a value in $\quant_{\frac{1}{2}}[\tilde{x}]$. Because $\quant^+_{\alpha}[\tilde{x}] \neq \quant^{-}_{\alpha}[\tilde{x}]$, in general, VaR also does not equal the median. A careful analysis shows that the weighted median algorithm in~\citet[section~17]{Korte2012} computes $\quant^-_{\frac{1}{2}}[\tilde{x}]$, which does not equal $\varo_{\frac{1}{2}}[\tilde{x}]$.
\end{remark}

\section{Computing \texorpdfstring{$\varphi$}{phi}-divergence Risk Measures in Linear Time} \label{sec:phi-divergence}

In this section, we describe and analyze QuickDivergence, a new linear-time framework for computing a class of $\varphi$-divergence risk measures, including CVaR and TVaR. QuickDivergence treats risk evaluation as a linear optimization over a polymatroid. Standard polymatroid optimization algorithms must sort their input, which requires $\Omega(n \log  n)$ runtime. We identify a structural condition of polymatroid risk measures that allows us to avoid sorting and instead use the linear-time QuickVaR algorithm described in \cref{sec:var-linear}.

The remainder of the section is organized as follows. \Cref{sec:quickdivergence} describes a general linear-time algorithm suitable for $\varphi$-divergence risk measures that satisfy an EWS property. Then \cref{sec:cvar-tvar-are} shows that both CVaR and TVaR satisfy the EWS property and can be computed in linear time. Note that we use a fixed distribution $\bm{p} \in \Delta_n$ and omit it in the notation when it is obvious from the context.

\subsection{QuickDivergence} \label{sec:quickdivergence}
We first recall the definition of a submodular function~\cite[section~2.1]{Korte2012} and a polymatroid~\cite[definition~14.9]{Korte2012}.
\begin{definition} \label{def:polymatroid}
A function $f \colon 2^{1{:}n} \to \Real_{\ge 0}$ is \emph{submodular} if
\[ 
  f(\mathcal{A}) + f(\mathcal{B})
  \; \geq\;
  f(\mathcal{A} \cup \mathcal{B}) + f(\mathcal{A} \cap \mathcal{B}),
  \quad
  \forall \mathcal{A}, \mathcal{B} \subseteq 1{:}n.
\]
For a submodular $f \colon 2^{1{:}n} \to \Real_{\ge 0}$, a \emph{polymatroid} $\mathcal{Q}_f$ is a polytope defined as
  \[
    \mathcal{Q}_f
    \;:=\;
    \left\{ \bm{q} \in \Real_{\ge 0}^n  \mid   \bm{1}\tr \bm{q}_{\mathcal{U}} \leq f(\mathcal{U}), \, \forall \,  \mathcal{U} \subseteq 1{:}n \right\},
\]
and a \emph{normalized polymatroid} is $\widehat{\mathcal{Q}}_f := \mathcal{Q}_f \cap \Delta_n $.
\end{definition}

Polymatroids are useful in optimization because any coherent, comonotonic monetary risk measure can be formulated as linear optimization over a polymatroid~\cite[corollary~4.89]{Follmer2016}. A risk measure $\rho$ is \emph{coherent} if it is concave and positively homogeneous.  A risk measure $\rho$ is \emph{comonotonic} if for any comonotone $\tilde{x}, \tilde{y}$: $\rho[ \tilde{x} + \tilde{y} ] = \rho[ \tilde{x} ] + \rho[ \tilde{y} ]$, where the random variables $\tilde{x}, \tilde{y} \in \Real^n$ are \emph{comonotone} if $\tilde{x}(i) \ge \tilde{x}(j) \Leftrightarrow \tilde{y}(i) \ge \tilde{y}(j)$ for each $i,j\in 1{:}n$.

\begin{proposition} \label{prop:coherent-polymatroid}
A monetary risk measure $\rho\colon \Delta_n \times  \Real^n \to \Real$ is a coherent comonotonic risk measure if and only if
\begin{equation}\label{eq:comonotonic-optimization}
  \rho[\tilde{x}] = \min_{\bm{q}\in \widehat{\mathcal{Q}}_c} \bm{q}\tr  \bm{x},
  \qquad
  \text{where}
  \qquad
  c(\mathcal{A}) := -\rho[ -\tilde{1}_{\mathcal{A}} ], \quad \forall  \mathcal{A} \subseteq 1{:}n,
\end{equation}
where $\tilde{1}$ is an indicator random variable. Moreover, if the risk measure is law-invariant, then there exists $g\colon \Real \to \Real$ such that $c(\mathcal{A}) = g(\bm{1}\tr \bm{p}_{\mathcal{A}})$ for each $\mathcal{A} \subseteq 1{:}n$.
\end{proposition}
\Cref{prop:coherent-polymatroid} follows  from corollary~4.88 and theorem~4.94 in \citet{Follmer2016}.

The standard solution to the polymatroid optimization problem in~\eqref{eq:comonotonic-optimization} is to use a greedy algorithm, such as \cref{alg:greedy-polymatroid}. This algorithm adapts the standard polymatroid optimization, such as \citet[section~14.9]{Korte2012}, to the normalized polymatroid $\widehat{\mathcal{Q}}_f$. See \cref{prop:standard-polymatroid-works} for the proof of correctness. Note that the greedy approach in \cref{alg:greedy-polymatroid} requires sorting and, therefore, runs in $\Omega(n \log  n)$ time.

\begin{algorithm}
    \caption{Linear minimization over polymatroids.} \label{alg:greedy-polymatroid}
    \Input{submodular monotone function $f \colon  2^{1{:}n} \to \Real_{\ge 0}$, and an objective $\bm{x}\in \Real^n$}
    \Output{$\hat{\bm{q}} \in \argmin_{\bm{q}\in \widehat{\mathcal{Q}}_f} \bm{x}\tr \bm{q}$}
    Use sort to get
    $\sigma \colon 1{:}n \to 1{:}n$ such that
    $i > j \implies x_{\sigma(i)} \ge  x_{\sigma(j)}, \; \forall i, j \in 1{:}n$\;
     \For{$k\in 1{:}n$}{
       $\hat{q}_{\sigma(k)} \gets f\left(\sigma(1{:}k)\right) - f\left(\sigma(1{:}(k-1))\right) $
    }
    \Return{$\hat{\bm{q}}$}
\end{algorithm}

We now define a special class of submodular functions, called EWS, that makes sorting unnecessary when optimizing a linear function over a polymatroid.
\begin{definition} \label{asm:submodular-function}
A function $f\colon 2^{1{:}n} \to \Real_{\ge 0}$ is $\ews(\bm{\mu} ,m, c)$, or \emph{Element-Wise-Separable}, for $c \in \Real_{\ge 0}$,  $m\in \Real_{>0}$, such that $c \le 1$, $c + m \ge 1$, and $\bm{\mu}\in \Delta_n$ if $f(\mathcal{A}) := g(\bm{1}\tr \bm{\mu}_{\mathcal{A}})$ for all $\mathcal{A}\subseteq 1{:}n$ and
\begin{equation} \label{eq:def-g}
    g(x) :=
    \begin{cases}        
    \min \{c + m  \cdot x, 1\} &\text{if } x > 0, \\
    0 &\text{otherwise}.
    \end{cases}
\end{equation}
\end{definition}
EWS functions satisfy the properties necessary to define a polymatroid as we show next.
\begin{proposition}
An EWS function $f\colon 2^{\Omega} \to  \Real_{\ge 0}$ is submodular, monotone, and satisfies that $f(\emptyset) = 0$, and $f(\Omega) = 1$.
\end{proposition}
\begin{proof} 
To show that $f$ is submodular, it is sufficient, from \cref{lem:submodular-concave}, to show that $g$ is concave on $[0,1]$. The function $g$ is concave on $(0, 1]$ because it is the minimum of two linear functions~\cite{boyd2004convex}. The concavity of $g$ on $[0, 1]$ follows from the definition. Since $g(0) = 0$, for each $\lambda \in (0,1]$ and $b \in ( 0, 1]$ we get that
\begin{align*}
  (1 - \lambda)g(0) + \lambda \cdot g(b)
  = \lambda \cdot g(b)
  = \min\left\{\lambda  \cdot c + m\cdot \lambda\cdot b, \lambda\right\} 
  \leq \min \left\{ c + m \cdot \lambda \cdot b, 1 \right\}
  = g\left( (1 - \lambda)\cdot 0 + \lambda\cdot b\right).
\end{align*}
Because the inequality also holds for $\lambda = 0$, $g$ is concave on $[0,1]$ by the definition. The monotonicity, non-negativity and the boundary value conditions for $f$ follow by algebraic manipulation from the assumptions on $c$ and
$m$.
\end{proof}

\begin{algorithm}
  \caption{\texttt{QuickDivergence}: Linear minimization over EWS polymatroids.} \label{alg:maximum-weight-independent}
  \Input{$\bm{x} \in \Real^n$, $\bm{p} \in \Delta_n$, $\ews(\bm{p}, m, c)$ function $f \colon 2^{1{:}n} \to \Real_{\ge 0}$}
  \Output{$\bm{q}\opt \in \argmin_{\bm{q} \in \widehat{\mathcal{Q}}_f} \bm{x}\tr\bm{q}$}
  Choose $k_{\min} \in \argmin_{i \in  1{:}n} x_i \; \operatorname{s.\,t.}\; p_i > 0$; break ties to choose the smallest $i$ \tcp{Analogous to $\ess\inf$}
    $v \gets
    \begin{cases}
      \operatorname{QuickVaR}(\bm{x}, \bm{p}, (1 - c)/m), &\text{ if } (1-c) / m < 1, \\
      \infty, &\text{ otherwise}
    \end{cases}
     $\;
  $\mathcal{K}_{< v} \gets \{ i \in  1{:}n \mid x_i < v \}$ ;~
  $\mathcal{K}_{= v} \gets \{ i \in 1{:}n \mid x_i = v \}$ \;
  $l \gets |\mathcal{K}_{< v} \cup \{ k_{\min} \} |$\;
  $\theta_l \gets  1 - \min \left\{ c + m\cdot \bm{1}\tr \bm{p}_{\mathcal{K}_{<v} \cup \{k_{\min}\} }, 1 \right\}$ \;
  \For{$i \in 1{:}n$}{
    \lIf{$i = k_{\min}$}{$q\opt_i \gets \min \left\{  m\cdot p_i + c, 1\right\}$ }
   \lElseIf{$i \in \mathcal{K}_{< v}$} {$q\opt_i \gets m \cdot p_i$}
   \lElseIf{$i \in \mathcal{K}_{=v}$} {$q\opt_i \gets \min \left\{  \theta_l, \, m \cdot p_i \right\} $;~ $\theta_{l+1}\gets \theta_l - q\opt_i;~ l \gets l+1$}
   \lElse{$q\opt_i \gets 0$}
  }
  \Return{$\bm{q}\opt$}
 \end{algorithm}

\cref{alg:maximum-weight-independent} describes \emph{QuickDivergence}, an algorithm for polymatroid optimization. QuickDivergence differs from the standard linear polymatroid algorithm in \cref{alg:greedy-polymatroid} in that it does not require sorting the input vector. Instead, the algorithm uses QuickVaR to partially sort the input vector.

The following theorem states the correctness and runtime of \cref{alg:maximum-weight-independent}. 
\begin{theorem}
\Cref{alg:maximum-weight-independent} solves~\eqref{eq:comonotonic-optimization} in expected $O(n)$ runtime.
\end{theorem}
\begin{proof}
To prove the correctness of \cref{alg:maximum-weight-independent}, we show that it returns $\bm{q}\opt$ that equals $\hat{\bm{q}}$ returned by the standard \cref{alg:greedy-polymatroid}.  Suppose that \cref{alg:greedy-polymatroid} chooses a permutation $\sigma\colon  1{:}n \to 1{:}n$ such that $i \le j \Rightarrow x_{\sigma(i)} \le x_{\sigma(j)}$ and $x_{\sigma(i)} = x_{\sigma(j)} \wedge i < j \Rightarrow \sigma(i) < \sigma(j)$ for each $i,j\in 1{:}n$. In other words, the permutation $\sigma$ needs to be generated by a stable sort algorithm. We emphasize that the stability assumption is only needed for our proof technique; the actual implementation does not require it.

We now prove that $q\opt_i = \hat{q}_i$ for each $i\in 1{:}n$ by separately analyzing the following four cases.

\emph{(i)} Suppose that $i = k_{\min}$. Because $p_i > 0$ from the construction of $k_{\min}$ then
\[
  q\opt_i = \min \left\{   m\cdot p_i + c, 1 \right\} = g(p_i) =  g(p_i) - g(0) = f(\{ i \} \cup \mathcal{J}) - f(\mathcal{J})  = \hat{q}_i,
\]
where $\mathcal{J} := \{ j \in 1{:}n \mid \sigma^{-1}(j) < \sigma^{-1}(i) \}$ and $\bm{1}_{\mathcal{J}}\tr \bm{p} = 0$ by the construction of $k_{\min}$.

\emph{(ii)} Suppose that $i \in \mathcal{K}_{<v} \wedge i \neq k_{\min}$. From \cref{thm:quantile-in-finite-omega}, we get for each $\mathcal{O} \subseteq \mathcal{K}_{<v}$ that
\begin{equation} \label{eq:lt-v-prob}
 \bm{1}\tr \bm{p}_{\mathcal{O}} \le \P{\tilde{x} < v} \le \frac{1-c}{m},
\end{equation}
and, hence, from~\eqref{eq:def-g}:
\(  f(\mathcal{O}) = g\left(\bm{1}\tr \bm{p}_{\mathcal{O}}\right) = c + m \cdot \bm{1}\tr \bm{p}_{\mathcal{O}},
  \text{ when }
  \bm{1}\tr \bm{p}_{\mathcal{O}} > 0.
\) 
Let $k := \sigma^{-1}(i)$ and note that $ \sigma(1{:}k) \subseteq \mathcal{K}_{< v} $ because if $l \in 1{:}k$, then $x_{\sigma(l)} \le x_{\sigma(k)} = x_i < v$ and $\sigma(l) \in \mathcal{K}_{< v}$. Then, since $i \neq k_{\min}$:
\[
 \hat{q}_i = \hat{q}_{\sigma(k)} = f(\sigma(1{:}k)) - f(\sigma(1{:}(k-1)))
= m \cdot \bm{1}\tr \bm{p}_{\sigma(1{:}k) \setminus \sigma(1{:}(k-1)) } = m\cdot p_{\sigma(k)} = m\cdot p_i = q\opt_i. 
\]
The third equality follows from $i \neq k_{\min}$ and $p_{k_{\min}} \neq  0$.

\emph{(iii)} Suppose that $i \in \mathcal{K}_{=v} \wedge i \neq k_{\min}$. Here, we use the assumption that $\sigma$ is a stable sort and prove the claim by induction on $l \in |\mathcal{K}_{<v} \cup \{k_{\min}\}|, \dots, |\mathcal{K}_{<v}\cup \{k_{\min}\}| + |\mathcal{K}_{=v} \setminus \{ k_{\min}\}|$ showing that
\begin{align}
 \label{eq:theta-equality-1}
  \theta_l &= 1 - f(\sigma(1{:}l)),  \\
  \label{eq:theta-equality-2}
    \hat{q}_{\sigma(l+1)} &= q\opt_{\sigma(l+1)}.
\end{align}
To prove the base case, suppose that $l = |\mathcal{K}_{<v} \cup \{k_{\min}\}|$. The equality in~\eqref{eq:theta-equality-1} holds by algebraic manipulation from~\eqref{eq:def-g} since $\mathcal{K}_{<v} \cup \{ k_{\min} \} = \sigma(1{:}l)$. The equality in~\eqref{eq:theta-equality-2} follows from~\eqref{eq:def-g} and~\eqref{eq:lt-v-prob}. First, assume that $\theta_l > 0$; then
\begin{equation} \label{eq:q-q-eq}
  \begin{aligned}
\hat{q}_{i} = \hat{q}_{\sigma(l+1)}  
  &= f ( \sigma(1{:}(l+1) )) - f ( \sigma(1{:}l) ) \\
  &= \min \left\{  1,\; c + m \cdot \bm{1}\tr \bm{p}_{\sigma(1{:}(l+1))} \right\} - c - m \cdot \bm{1}\tr \bm{p}_{\sigma(1{:}l)} \\
  &= \min \left\{1  - c - m \cdot \bm{1}\tr \bm{p}_{\sigma(1{:}l)}, \; m \cdot p_{\sigma(l+1)}\right\} \\
  &= \min \{ \theta_l, m \cdot p_i \} = q_i\opt.
  \end{aligned}
\end{equation}
When $\theta_l = 0$, we get by algebraic manipulation that $\hat{q}_i = 0 = q\opt_i$.

To prove the inductive case, suppose that~\eqref{eq:theta-equality-1} and~\eqref{eq:theta-equality-2} hold for $l$, and we prove it for $l+1$. The equality in~\eqref{eq:theta-equality-1} can be derived from~\eqref{eq:theta-equality-2} and the construction in \cref{alg:greedy-polymatroid} as
\begin{align*}
  \theta_{l+1}
  = \theta_l - q\opt_{\sigma(l+1)} 
    = \theta_l - \hat{q}_{\sigma(l+1)}
    = 1 - f(\sigma(1{:}l)) - \left(  f(\sigma(1{:}(l+1))) - f(\sigma(1{:}l)) \right) 
  = 1 - f(\sigma(1{:}(l+1))).
\end{align*}
The equality in~\eqref{eq:theta-equality-2} follows analogously to~\eqref{eq:q-q-eq} if $\theta_{l+1} > 0$:
\begin{equation} \label{eq:q-q-eq-2}
  \begin{aligned}
\hat{q}_{i} = \hat{q}_{\sigma(l+2)}
  &= f ( \sigma(1{:}(l+2) )) - f(\sigma(1{:}(l+1))) \\
  &= \min \left\{  1,\; c + m \cdot \bm{1}\tr \bm{p}_{\sigma(1{:}(l+2))} \right\} - (1-\theta_{l+1}) \\
  &= \min \left\{\theta_{l+1}, \;c + m \cdot \bm{1}\tr \bm{p}_{\sigma(1{:}(l+2))} - f(\sigma(1{:}(l+1))) \right\} \\
  &= \min \left\{\theta_{l+1}, \; m \cdot p_{\sigma(l+2)}  \right\} 
  = q_i\opt.
  \end{aligned}
\end{equation}
If $\theta_{l+1} = 0$, we get by algebraic manipulation that $\hat{q}_i = 0 = q\opt_i$.

\emph{(iv)} Suppose that $i \in \mathcal{K}_{>v} := \{ j \in 1{:}n \mid x_j > v \}$ and let $k = \sigma^{-1}(i)$. Then, using~\eqref{eq:var-property-prob} we have that:
\begin{align*}
  1
  \ge f(\sigma(1{:}k))
  \ge f(\sigma(1{:}(k-1))) 
  \ge \min \{c + m \cdot \bm{1}\tr \bm{p}_{\sigma(1{:}(k-1))}, 1\}
  \ge 1,
\end{align*}
since $\bm{1}\tr\bm{p}_{\sigma(1{:}(k-1))} > \frac{1-c}{m}$ from \cref{thm:quantile-in-finite-omega} and the construction of $v$ in \cref{alg:maximum-weight-independent}. Then
\[
  \hat{q}_i = \hat{q}_{\sigma(k)}
  =  f(\sigma(1{:}k)) - f(\sigma(1{:}(k-1)))
  = 0 = q\opt_i.
\]

\end{proof}

\subsection{Solving CVaR and TVaR in Linear Time} \label{sec:cvar-tvar-are}

We show that CVaR and TVaR are two examples of $\varphi$-divergence risk measures that can be formulated as linear optimization over a polymatroid with an EWS capacity function and, therefore, can be solved in linear time. To formulate the CVaR polymatroid optimization, we define the capacity function $f_{\alpha}^{\cvaro}\colon 2^{\Omega } \to \Real$  for each $\mathcal{O} \subseteq \Omega$ and $\alpha \in [0,1]$ as
\begin{equation} \label{eq:f-cvar}
  f_{\alpha}^{\cvaro}(\mathcal{O})
    \;:=\;
      \begin{cases}
        \min\left\{\frac{1}{\alpha} \cdot \P{\mathcal{O}}, 1\right\} & \text{if } \alpha \in (0,1], \\
        0 & \text{if } \alpha = 0 \wedge \P{\mathcal{O}} = 0,  \\
        1 & \text{if } \alpha = 0 \wedge \P{\mathcal{O}} > 0. 
      \end{cases} 
\end{equation}
The Choquet capacity function $f_{\alpha }^{\cvaro}$ for CVaR defined in~\eqref{eq:f-cvar} is well known, but is usually restricted to $\alpha\in (0,1)$~\cite{acerbi2002spectral}. The following proposition states its correctness.
\begin{proposition} \label{prop:cvar-ews-correct}
For each $\alpha\in [0,1]$, $\tilde{x} \colon \Omega \to \Real$, and $\bm{p} \in \Delta_n$: 
\begin{equation} \label{eq:cvar-correct}
  \cvar{\alpha }{\tilde{x} }
  \;=\;
  \min_{\bm{q} \in \widehat{\mathcal{Q}}_{f_{\alpha }^{\cvaro}}} \bm{x}\tr \bm{q}.
\end{equation}
Moreover, $f^{\cvaro}_{\alpha }$ is $\ews(\bm{p}, 1/\alpha, 0)$ when $\alpha > 0$, and $\ews(\bm{p}, 1, 1)$ when $\alpha = 0$.
\end{proposition}

To compute TVaR, we define the capacity function $f_{\alpha}^{\tvaro}\colon 2^{\Omega } \to \Real$  for each $\mathcal{O} \subseteq \Omega$ and $\alpha \in [0,1]$ as
\begin{equation}\label{eq:f-tvar}
  f^{\tvaro}_{\alpha}(\mathcal{O})
    \; :=\; 
      \begin{cases}
        \min\left\{\P{\mathcal{O}} + \min \left\{ \sqrt{\frac{1}{2} \log \frac{1}{\alpha}}, 1\right\}, 1\right\} &\text{if } \alpha \in (0,1] \wedge \P{\mathcal{O}} > 0, \\
        0 & \text{if } \P{\mathcal{O}} = 0, \\
        1 & \text{if } \alpha = 0 \wedge \P{\mathcal{O}} > 0.
      \end{cases}
\end{equation}
TVaR was first proposed in the context of $\varphi$-divergence risk measures~\cite{Kruse2021} as well as in the context of robust stochastic programming~\cite{Shapiro2017}. TVaR is also very closely related to the $L_1$ ball used in robust Markov decision processes~\cite{Ho2021,Ho2022}. However, we are not aware of prior work that defines or uses the capacity function in~\eqref{eq:f-tvar}. The following proposition shows that this function can be used to correctly formulate the TVaR evaluation using an EWS function. 
\begin{proposition} \label{prop:tvar-ews-correct}
For each $\alpha\in [0,1]$, $\tilde{x} \colon \Omega \to \Real$, and $\bm{p} \in \Delta_n$: 
\begin{equation} \label{eq:tvar-correct}
  \tvar{\alpha}{\tilde{x} } = \min_{\bm{q}\in \widehat{\mathcal{Q}}_{f_{\alpha}^{\tvaro}}} \bm{x}\tr \bm{q}.
\end{equation}
Moreover, $f^{\tvaro}_{\alpha }$ is $\ews(\bm{p}, 1, \min \left\{   \sqrt{\frac{1}{2}  \log \frac{1}{\alpha}}, 1 \right\})$  when $\alpha > 0$, and $\ews(\bm{p}, 1, 1)$ when $\alpha = 0$.
\end{proposition}

\section{Experiments}\label{sec:experiments}

\begin{figure}
  \includegraphics[width=0.5\textwidth]{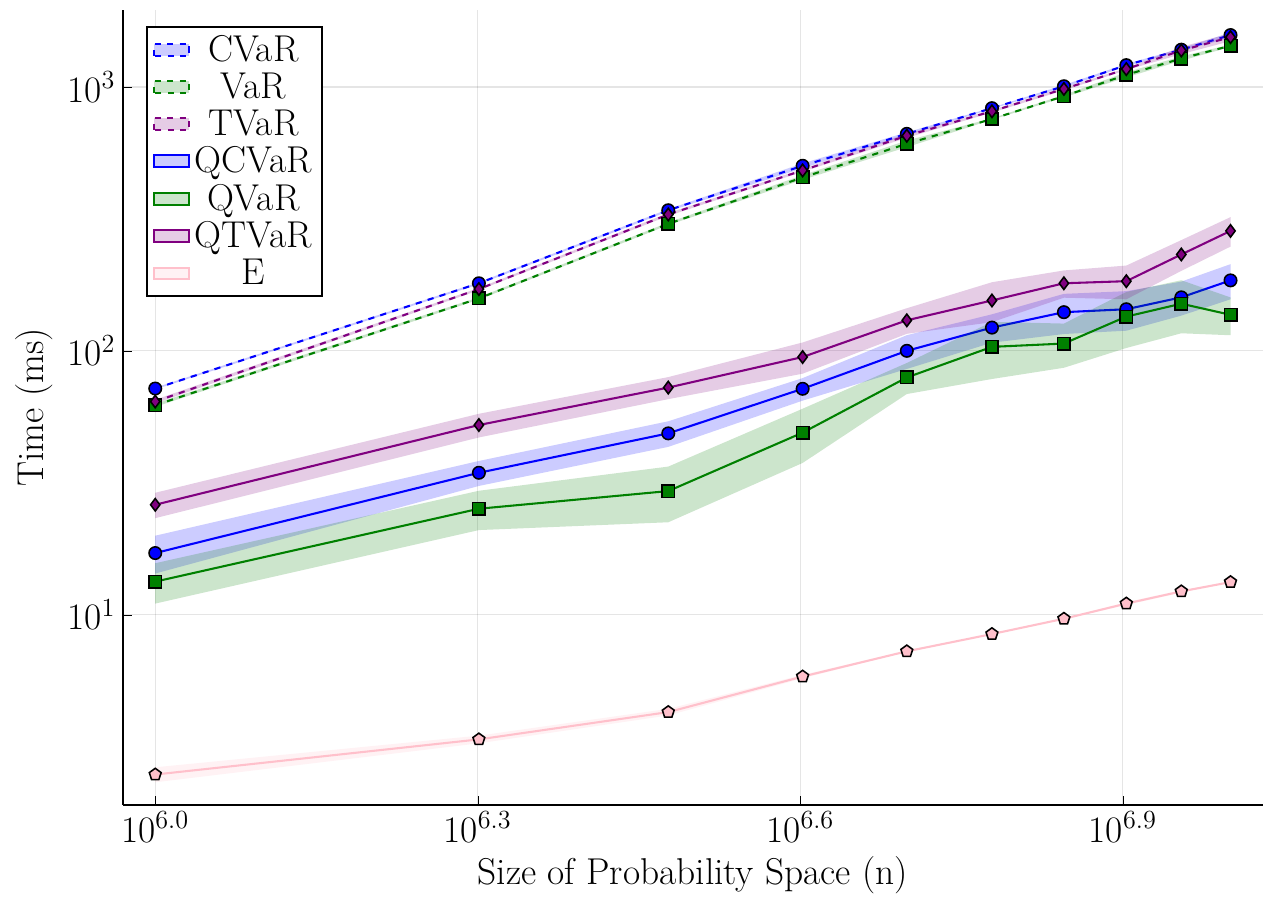}
  \includegraphics[width=0.5\textwidth]{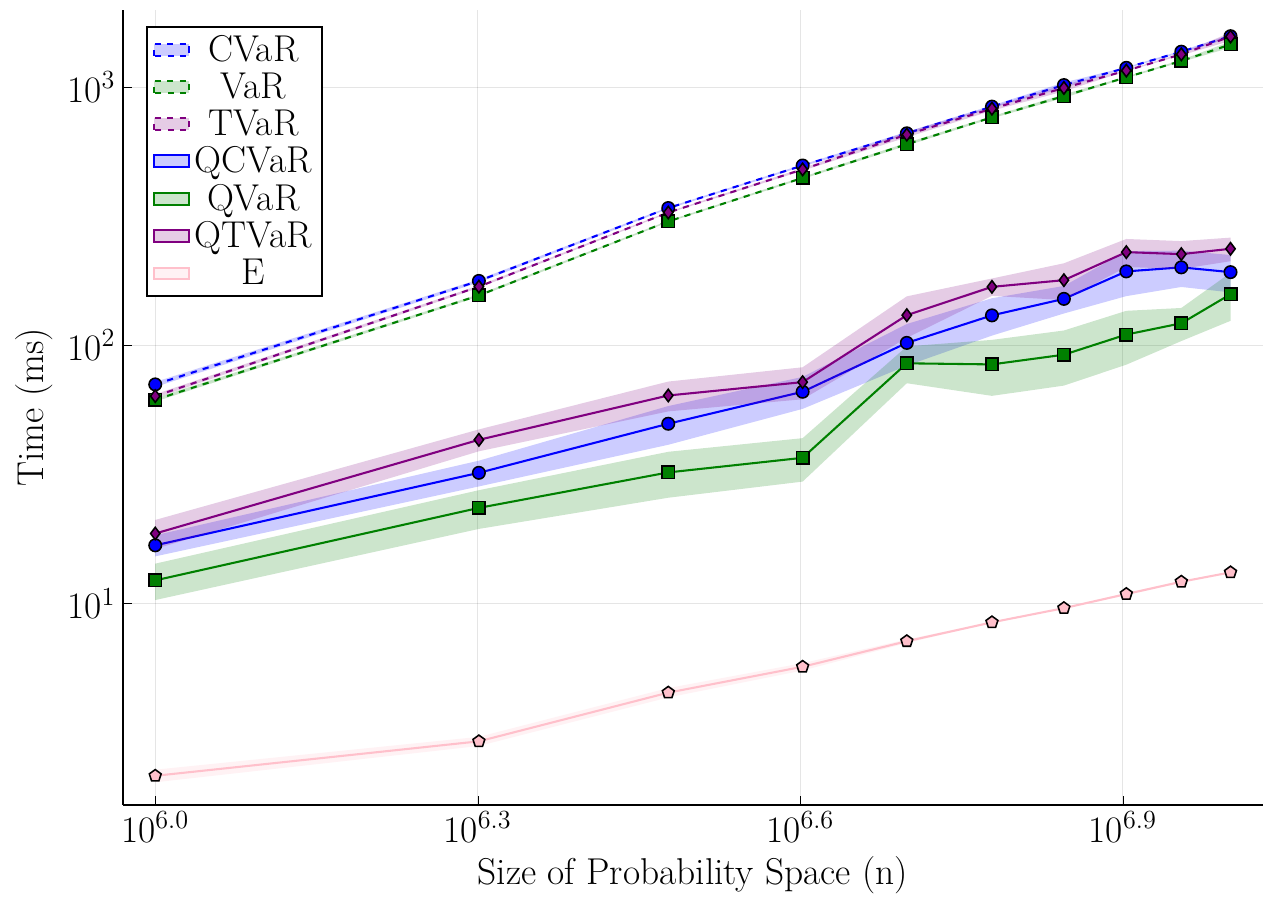}
  \caption{Comparison of runtime of the algorithms for samples from a \emph{uniform} random variable (left) and a \emph{sparse} random variable (right). Shaded regions denote $95\%$ confidence bands over the $10$ runs.}
  \label{fig:benchmarks}
\end{figure}

We evaluate the proposed algorithms empirically to confirm that their theoretical performance guarantees translate into faster practical runtimes. We compare each linear-time algorithm against its standard sorting-based counterpart. Throughout this section, we prefix the name of our linear-time implementation with a ``Q'' to distinguish it from the sorting-based baseline, giving QVaR, QCVaR, and QTVaR. We also report the runtime of computing the standard expectation $\mathbb{E}$ as a baseline.

Our implementation, together with the code and data needed to reproduce all results, is available at \url{https://github.com/marekpetrik/paper_lineartime_risk}, and the actual implementation of the risk measures can be found at \url{https://github.com/RiskAverseRL/RiskMeasures.jl}. All experiments were run on an AMD Ryzen Threadripper 3970X with 32~cores and 64~threads at $4.55$\,GHz with $256$\,GB of RAM using Julia~1.12.5. To reduce noise in the timing measurements, we disable the Julia garbage collector during each run.

We begin with synthetic data drawn from a uniform distribution with $n$ samples, varying $n$ from $10^6$ to $10^7$. For each value of $n$ we run every algorithm $10$ times on independently generated samples and report the mean runtime. \Cref{fig:benchmarks} shows the results on a log-log scale for random variables with a uniform (left) and sparsely-supported (right) distributions. We fix the risk level to $\alpha = 0.95$, though the results are similar for other levels except the degenerate cases $\alpha = 0$ and $\alpha = 1$. The shaded regions denote $95\%$ confidence bands over the $10$ runs; the bands are tight because the measured runtimes vary little between runs.

The results confirm that the quick algorithms are substantially faster than their sorting-based counterparts, with the gap widening as $n$ grows. This behavior matches the theoretical analysis of \cref{sec:var-linear,sec:phi-divergence}: the sorting-based methods require $O(n \log n)$ operations, whereas our algorithms run in $O(n)$ expected time, so the logarithmic factor separating them grows with the size of the probability space. The acceleration is most consequential for VaR and CVaR, which are the most widely used risk measures in the applications discussed in \cref{sec:intro} and are frequently evaluated on large random variables.

\begin{table}
\centering
\caption{Mean runtime of algorithms in the stock market scenario.}
\label{tab:stock}
\begin{tabular}{lrr}
\toprule
Method & Mean (ms) & St. dev. (ms) \\
\midrule 
VaR & 0.37 & 0.00 \\ 
QVaR & 0.13 & 0.02 \\
\midrule 
CVaR & 0.39 & 0.00 \\ 
QCVaR & 0.22 & 0.03 \\ 
\midrule 
TVaR & 0.40 & 0.01 \\ 
QTVaR & 0.25 & 0.05 \\ 
\midrule 
$\mathbb{E}$ & 0.01 & 0.00 \\
\bottomrule
\end{tabular}
\end{table}

\begin{table} 
  \centering
\caption{Mean runtime (ms) of algorithms on smaller sparse random domains.}
\begin{tabular}{rrrrrrrr}
\toprule
n & $\mathbb{E}$ & VaR & QVaR & CVaR & QCVaR & TVaR & QTVaR \\
\midrule
1000 & 0.00 & 0.03 & 0.01 & 0.03 & 0.02 & 0.03 & 0.02\\
2000 & 0.00 & 0.06 & 0.02 & 0.06 & 0.03 & 0.06 & 0.04\\
3000 & 0.01 & 0.09 & 0.04 & 0.10 & 0.05 & 0.10 & 0.07\\
4000 & 0.00 & 0.13 & 0.04 & 0.14 & 0.06 & 0.13 & 0.08\\
5000 & 0.01 & 0.16 & 0.05 & 0.17 & 0.10 & 0.18 & 0.10\\
6000 & 0.00 & 0.20 & 0.08 & 0.21 & 0.10 & 0.21 & 0.11\\
7000 & 0.00 & 0.23 & 0.08 & 0.25 & 0.12 & 0.24 & 0.13\\
8000 & 0.00 & 0.27 & 0.09 & 0.28 & 0.13 & 0.28 & 0.15\\
9000 & 0.01 & 0.30 & 0.09 & 0.32 & 0.12 & 0.31 & 0.17\\
10000 & 0.01 & 0.33 & 0.11 & 0.35 & 0.16 & 0.35 & 0.19\\
\bottomrule
\end{tabular}
\label{tab:small_results}
\end{table}

To confirm the speedups persist beyond large synthetic inputs, we repeat the comparison on two more realistic settings. The first one is a Stock Market scenario, in which the random variable models the returns of a financial portfolio. \Cref{tab:stock} reports the mean runtime and standard deviation of each method. The second one is small-scale problems, where the asymptotic advantage is least pronounced but can be practically important when the risk measures must be computed repeatedly. \Cref{tab:small_results} reports the mean runtime on sparse random variables over probability spaces of size $n$ ranging from $1000$ to $10\,000$, using the same protocol as above but with $50$ trials per configuration in place of $10$. Even at these sizes, the quick algorithms outperform their sorting-based counterparts for every method and every value of $n$, confirming that the linear-time algorithms deliver consistent improvements across the full range of input sizes without incurring any penalty on small inputs.

\section{Conclusion}

We propose and analyze new algorithms for computing several common risk measures, including VaR, CVaR, and TVaR, in $O(n)$ expected time, where $n$ is the size of the probability space. This compares favorably with the $O(n \log n)$ time required by sorting-based methods.  This speedup is significant because risk measures are frequently evaluated on large random variables in applications such as robotics, finance, and infrastructure maintenance. An important direction for future work is to develop fast algorithms for other risk measures, such as EVaR and Expectile.

\bibliography{fastcvar.bib}

\appendix
\crefalias{section}{appendix}
\crefalias{Section}{Appendix}

\section{Supplemental Results for Section~\ref{sec:var-linear}} \label{sec:stand-var-algor}

In this section, we summarize the standard and supporting algorithms for solving VaR. We also include the selected technical proofs for the correctness of the proposed algorithm. 

\begin{lemma} \label{lem:prob-lumpy}
For each $t\in \Real$:
  \[
    \P{\tilde{x} = t} > 0
    \implies
    \exists \omega \in \Omega, x(\omega) = t.
  \]
\end{lemma}
\begin{proof}
The result follows by induction on the size $|\Omega|$, dropping the constraint $\bm{1}\tr \bm{p} = 1$. 
\end{proof}

\begin{algorithm}
\caption{Standard VaR algorithm}
\label{alg:slow-quantile}
\Input{$\bm{x}\in \Real^n$, $\bm{p}\in \Delta_n$, and $\alpha \in [0,1)$}
\Output{$\varo_{\alpha}[\tilde{x}]$}
Sort $\bm{x}$ indices to get $\sigma \colon 1{:}n \to 1{:}n$ such that $i > j \implies x_{\sigma(i)} \ge  x_{\sigma(j)}, \; \forall i, j \in 1{:}n$\;
$q_0 \gets 0$
\tcp*{$q_i = \P{\tilde{x} \le x_{\sigma (i)}}$ when $x_{\sigma(i)}$ unique}
\For{$i\in 1{:}n$} {
    $q_i \gets q_{i-1} + p_{\sigma(i)}$ \;
    \lIf{$q_i > \alpha$} {\KwRet{$x_{\sigma(i)}$}}
}
\end{algorithm}

\begin{proposition}
For each $\bm{x} \in \Real^n$, $\bm{p}\in \Delta_n$, and $\alpha\in [0,1)$, \cref{alg:slow-quantile} returns $\var{\alpha}{\tilde{x}}$.
\end{proposition}
\begin{proof}
  First, one can readily show by induction that for each $i\in 1{:}n$ (note that values of $\bm{x}$ need not be unique):
\begin{equation} \label{eq:proof-q-j}
  \P{\tilde{x} < x_{\sigma(i)}} \le q_{i-1} \le q_i \le \P{\tilde{x} \le x_{\sigma(i)}}.
\end{equation}
Since $\alpha < 1 = q_n$, the algorithm is guaranteed to terminate with some $i\opt \in 1{:}n$ and return $x_{\sigma(i\opt)}$. Therefore, $q_{i\opt-1} \le \alpha < q_{i\opt}$ and applying~\eqref{eq:proof-q-j} gives us that
\[
 \P{\tilde{x} < x_{\sigma(i\opt)}} \le q_{i\opt-1} \le \alpha < q_{i\opt} \le  \P{\tilde{x} \le x_{\sigma(i\opt)}},
\]
and the result holds by \cref{thm:quantile-in-finite-omega}.
\end{proof}

\begin{algorithm}
  \caption{\texttt{partition}: Dutch national flag algorithm~\cite[chapter~14]{Dijkstra1976}}
  \label{alg:partition}
  \Input{$\bm{x} \in \Real^n$, $\bm{p} \in \Delta_n$, pivot $v\in \Real$, and front and back $f, b \in 1{:}n$, $f \le b$}
  \Output{$l, g \in f{:}b$, $\bm{x} \in \Real^n$, $\bm{p} \in \Delta_n$ such that \cref{lem:hoare-partition} holds}
  $l_0 \gets e_0 \gets f$; $g_0 \gets b$; $k \gets 0$ \;
  \While{$e_k \leq g_k$}
  {
  \label{line:partition-while}
    \If{$x_{e_k} < v$}{
      swap: $x_{e_k} \leftrightarrow x_{l_k}$ and $p_{e_k} \leftrightarrow p_{l_k}$\; \label{line:swap1}
      $l_{k+1} \gets l_k + 1$ ;
      $e_{k+1} \gets e_k + 1$\; \label{line:partition-e-1-1}
    }
    \ElseIf{$x_{e_k} > v$}{
      swap: $x_{e_k}\leftrightarrow x_{g_k}$ and $p_{e_k} \leftrightarrow p_{g_k}$\; \label{line:swap2}
      $g_{k+1} \gets g_{k} - 1$\; \label{line:partition-g-1}
    }
    \lElse {$e_{k+1} \gets e_{k} + 1$} \label{line:partition-e-1-2}
    $k \gets k + 1$
  }
  \KwRet{$(l_k, g_k, \bm{x}, \bm{p})$}\; 
\end{algorithm}

\begin{algorithm}
  \caption{\texttt{QuickVaR}: Non-recursive adaptation of \cref{alg:quick-quantile}} \label{alg:quickvar-nonrecursive}
  \Input{vector $\bm{x}^0$, pmf $\bm{p}^0$ and $\alpha_0 \in [0,1)$}
  \Output{$\varo_{\alpha}[\tilde{x}]$}
  $f_0 \gets 1$;~~~~$b_0 \gets n$;~~~~$k\gets 1$ \;
  \While{$b_{k-1} > f_{k-1}$} {
    $v_k \gets x_{\operatorname{rand}(f_{k-1}{:}b_{k-1})}$ \tcp*{Pivot} 
     $l_k, g_k, \bm{x}^k, \bm{p}^k \gets \operatorname{partition}(\bm{x}^{k-1}, \bm{p}^{k-1}, v_k, f_{k-1}, b_{k-1})$
    \tcp*{Alg.~\ref{alg:partition}, \cref{lem:hoare-partition}} \label{line:partion_line}    
    $t_k \gets \sum_{j=f_{k-1}}^{l_k-1} p^k_j$  \tcp*{$\P{\tilde{x}^k < v_k}$}
    $e_k \gets \sum_{j = l_k}^{g_k} p^{k}_j $   \tcp*{$\P{\tilde{x}^k = v_k}$}
    \lIf{$\alpha_{k-1} < t_k$\label{line:determine_left}}
    {
      $f_k \gets f_{k-1}$;~$b_k \gets l_k-1$;~
      $\alpha_k \gets \alpha_{k-1}$
    }
    \lElseIf{$t_k + e_k \le  \alpha_{k-1} \label{line:determine_right}$} 
    {
      $f_k \gets g_k+1$;~$b_k \gets b_{k-1}$;~$\alpha_k \gets \alpha_{k-1} - t_k - e_k$
    }
    \lElse{
      $f_k \gets g_k$;~~~$b_k \gets g_k$
    }
    $k \gets k + 1$\; 
  }
  \KwRet{$x^{k-1}_{b_{k-1}}$}\;
\end{algorithm}

The following lemma proves the properties of the $\operatorname{partition}$ operator defined in \cref{alg:partition}. 
\begin{lemma} \label{lem:hoare-partition}
If $(l, g, \bm{x}', \bm{p}') = \operatorname{partition}(\bm{x}, \bm{p}, v, f, b)$ for $f \le b$, then:
\begin{align} \label{eq:hoare-correspond}
  \bm{x}_{f{:}(l-1)}' < v, \qquad
  \bm{x}_{l{:}g}' = v,  \qquad
  \bm{x}_{(g+1){:}b}' > v.
\end{align}
Moreover, $ \bm{x}_{f{:}b} = \bm{x}'_{\sigma(f{:}b)}, \bm{p}_{f{:}b} = \bm{p}'_{\sigma(f{:}b)}$ for some permutation $\sigma\colon f{:}b \to f{:}b$, and $\bm{x}_{1{:}(f-1)} = \bm{x}'_{1{:}(f-1)}, \bm{p}_{1{:}(f-1)} = \bm{p}'_{1{:}(f-1)}$, and $\bm{x}_{(b+1):n} = \bm{x}'_{(b+1):n}, \bm{p}_{(b+1):n} = \bm{p}'_{(b+1):n}$. The elements outside of the range $f{:}b$ are not modified. The algorithm runs in $O(b - f + 1)$ time.
\end{lemma}
\begin{proof}[Proof of \cref{lem:hoare-partition}]
Let $v$ be an input to \cref{alg:partition} and consider the following invariants.
\begin{align}
\label{eq:ho-inv-less}   x_j &< v, & j\in f_k{:}(l_k-1), \\
\label{eq:ho-inv-equal}  x_j &= v, & j \in l_k{:}(e_k-1), \\
\label{eq:ho-inv-more}   x_j &> v, & j \in (g_k+1){:}b_k.
\end{align}
  
\emph{Initialization}. For $k=0$, conditions~\eqref{eq:ho-inv-less},~\eqref{eq:ho-inv-equal},~\eqref{eq:ho-inv-more} are true because the pertinent intervals are empty. 
  
\emph{Maintenance}. At each iteration of the loop that starts on line \ref{line:partition-while} of \cref{alg:partition}, there are 3 possible branches determined by the conditional logic: $x_{e_k} < v$, $x_{e_k} > v$, and $x_{e_k} = v$. Since $v, x_{e_k} \in \Real$ one of these branches must be executed for each iteration $k$.

Each time the $x_{e_k} < v$ or $x_{e_k} > v$ branch is executed, conditions~\eqref{eq:ho-inv-less} and~\eqref{eq:ho-inv-more}, respectively, remain true since in each case we swap the necessary elements to preserve each condition. Notice also that $e_k$ is incremented alongside $l_k$ to maintain $e_k - l_k = e_{k+1} - l_{k+1}$ as is necessary for condition~\eqref{eq:ho-inv-equal}.

When $x_{e_k} = v$, $e_k$ is incremented; since $e_0 = l_0$, this creates a difference between $e_{k+1}$ and $l_{k+1}$. The algorithm terminates when $e_t = g_t + 1$, so $g_t - l_t = e_t - l_t - 1 \ge 0$ and~\eqref{eq:ho-inv-equal} remains true.   
  
\emph{Termination}. In the loop on line \ref{line:partition-while} there are 3 possible paths of execution determined by the conditional logic $x_{e_k} < v$, $x_{e_k} > v$, and $x_{e_k} = v$. Each of these paths of execution decreases $g_k - e_k$ by 1. This occurs on lines \ref{line:partition-e-1-1} and \ref{line:partition-e-1-2} by increasing $e_k$ by 1 and keeping $g_k$ the same, or on line \ref{line:partition-g-1} by decreasing $g_k$ by 1 and keeping $e_k$ the same.
  
Observe that  $g_0 - e_0 = b - f$ is non-negative since $f\le b$. The difference $g_k-e_k$ decreases by 1 at each iteration, and each iteration requires constant time to complete. Because the algorithm returns when $g_t - e_t = -1$, it finishes in $O(b-f)$ time. 

On the last iteration $g_t = e_t$, therefore by the same reasoning as above $g_{t+1} - e_{t+1} = -1$ and condition~\eqref{eq:ho-inv-equal} holds for $g_{t+1} = e_{t+1} - 1$.
\end{proof}

\section{Supplemental Results for Section~\ref{sec:phi-divergence}}

The following proposition states the correctness of the standard polymatroid optimization algorithm when applied to a normalized polymatroid. 
\begin{proposition} \label{prop:standard-polymatroid-works}
Suppose that $f\colon  2^{1{:}n} \to \Real$ is submodular, monotone, $f(\emptyset) = 0$ and $f(1{:}n) = 1$, and let $\bm{x}\in \Real^n_{\ge 0}$. Then, \cref{alg:greedy-polymatroid} computes $\hat{\bm{q}}$ such that
  \begin{equation} \label{eq:normalized-poly-opt}
    \hat{\bm{q}}  \in  \argmin_{\bm{q}\in \widehat{\mathcal{Q}}_f}  \bm{x}\tr \bm{q}.
\end{equation}
\end{proposition}
\begin{proof}
First, we show that $\hat{\bm{q}} \in \widehat{\mathcal{Q}}_f$. Because the construction of $\hat{\bm{q}}$ in \cref{alg:greedy-polymatroid} matches the construction in proposition~14.10 in \citet{Korte2012}, since $f(\emptyset) = 0$ and $\hat{q}_{\sigma(1)} = f(\left\{ \sigma(1) \right\}) - f(\emptyset) = f(\left\{ \sigma(1) \right\})$, we have that $\hat{\bm{q}}  \in \mathcal{Q}_f$. In addition, 
  \[
   \sum_{i=1}^n \hat{q}_i =  
   \sum_{i=1}^n \hat{q}_{\sigma (i)} = f(1{:}n) - f(\emptyset) = 1, 
  \]
which implies that $\hat{\bm{q}} \in \Delta_n \cap \mathcal{Q}_f = \widehat{\mathcal{Q}}_f$, $\widehat{\mathcal{Q}}_f \neq \emptyset$, and the minimum~\eqref{eq:normalized-poly-opt} exists.

To show that $\hat{\bm{q}} $ is optimal, we reduce the optimization in~\eqref{eq:normalized-poly-opt} to the standard linear polymatroid maximization problem with non-negative weights~\cite[section 14.9]{Korte2012} as follows:
\begin{align*}
  \min_{\bm{q}\in \widehat{\mathcal{Q}}_f}  \bm{x}\tr \bm{q}
  &= \min_{\bm{q}\in \widehat{\mathcal{Q}}_f}  (\bm{x} - x_{\max}\cdot \bm{1}) \tr \bm{q} +  x_{\max}\\
  &= -\max_{\bm{q}\in \widehat{\mathcal{Q}}_f}  (x_{\max}\cdot \bm{1} - \bm{x}) \tr \bm{q} + x_{\max}\\
  &\ge -\max_{\bm{q}\in \mathcal{Q}_f}  (x_{\max}\cdot \bm{1} - \bm{x}) \tr \bm{q} + x_{\max}\\
  &= - (x_{\max}\cdot \bm{1} - \bm{x}) \tr \hat{\bm{q}} + x_{\max} \\
  &=  \bm{x}\tr \hat{\bm{q}}.
\end{align*}
Here, $x_{\max} := \max_{i\in 1{:}n} x_i $ and the equalities follow because $\widehat{\mathcal{Q}}_f$ is a normalized polymatroid and $\bm{1}\tr \hat{\bm{q}} = 1$. The second-to-last equality follows from theorem~14.11 in \citet{Korte2012} and the construction in the greedy polymatroid algorithm.
\end{proof}

The following lemma is folklore in submodular optimization~\cite[theorem~4.5.1]{Bilmes18}. We state and prove it for completeness.
\begin{lemma} \label{lem:submodular-concave}
  Suppose that $\bm{\mu}\in \Real^n_{\ge 0}$ and $g \colon \Real_{\ge 0} \to \Real$ is concave and monotone. Then $f(\mathcal{A}) := g(\bm{1}\tr\bm{\mu}_{\mathcal{A}})$ is a submodular monotone function. 
\end{lemma}
\begin{proof}
  The fact that $f$ is monotone follows from the composition of two monotone functions. To prove that $f$ is submodular, we use the following properties that hold for each $\mathcal{A}, \mathcal{B} \subseteq 1{:}n$:
  \begin{align}
    \label{eq:sets-bounds-nneg}
    \bm{1}\tr\bm{\mu}_{\mathcal{A} \cap \mathcal{B}}
    &\le \bm{1}\tr\bm{\mu}_{\mathcal{A}} \le \bm{1}\tr\bm{\mu}_{\mathcal{A} \cup \mathcal{B}},
    &
    \bm{1}\tr\bm{\mu}_{\mathcal{A} \cap \mathcal{B}}
    &\le \bm{1}\tr\bm{\mu}_{\mathcal{B}} \le \bm{1}\tr\bm{\mu}_{\mathcal{A} \cup \mathcal{B}}, \\
    \label{eq:mu-modular}
    \bm{1}\tr\bm{\mu}_{\mathcal{A}} + \bm{1}\tr\bm{\mu}_{\mathcal{B}}
    &= \bm{1}\tr\bm{\mu}_{\mathcal{A} \cup \mathcal{B}} + \bm{1}\tr\bm{\mu}_{\mathcal{A} \cap \mathcal{B}}.
  \end{align}
The inequalities in~\eqref{eq:sets-bounds-nneg} hold because $\bm{\mu} \ge  \bm{0}$ and the equality in~\eqref{eq:mu-modular} holds because $\mathcal{A} \mapsto \bm{1}\tr\bm{\mu}_{\mathcal{A}}$ is modular. Suppose that $\bm{1}\tr \bm{\mu}_{\mathcal{A} \cap \mathcal{B}} < \bm{1}\tr\bm{\mu}_{\mathcal{A} \cup \mathcal{B}}$; otherwise the submodularity condition in~\eqref{eq:submod-proof-cond} below holds with equality.  Then, writing $\bm{1}\tr\bm{\mu}_{\mathcal{A}}$ and $\bm{1}\tr\bm{\mu}_{\mathcal{B}}$ as a convex combination of $\bm{1}\tr\bm{\mu}_{\mathcal{A}\cap \mathcal{B}}$ and $\bm{1}\tr\bm{\mu}_{\mathcal{A} \cup \mathcal{B}}$ using~\eqref{eq:sets-bounds-nneg} and the concavity of $g$ we get that
  \begin{align}
  \label{eq:ga-bound}
    g(\bm{1}\tr\bm{\mu}_{\mathcal{A}})
    &\ge
      \frac{\bm{1}\tr\bm{\mu}_{\mathcal{A} \cup \mathcal{B}} - \bm{1}\tr\bm{\mu}_{\mathcal{A}}}{\bm{1}\tr\bm{\mu}_{\mathcal{A} \cup \mathcal{B}} - \bm{1}\tr\bm{\mu}_{\mathcal{A} \cap \mathcal{B}}} g(\bm{1}\tr\bm{\mu}_{\mathcal{A} \cap \mathcal{B}}) + 
      \frac{\bm{1}\tr\bm{\mu}_{\mathcal{A}} - \bm{1}\tr\bm{\mu}_{\mathcal{A} \cap \mathcal{B}}}{\bm{1}\tr\bm{\mu}_{\mathcal{A} \cup \mathcal{B}} - \bm{1}\tr\bm{\mu}_{\mathcal{A} \cap \mathcal{B}}} g(\bm{1}\tr\bm{\mu}_{\mathcal{A} \cup \mathcal{B}}), \\
  \label{eq:gb-bound}
    g(\bm{1}\tr\bm{\mu}_{\mathcal{B}})
    &\ge
      \frac{\bm{1}\tr\bm{\mu}_{\mathcal{A} \cup \mathcal{B}} - \bm{1}\tr\bm{\mu}_{\mathcal{B}}}{\bm{1}\tr\bm{\mu}_{\mathcal{A} \cup \mathcal{B}} - \bm{1}\tr\bm{\mu}_{\mathcal{A} \cap \mathcal{B}}} g(\bm{1}\tr\bm{\mu}_{\mathcal{A} \cap \mathcal{B}}) + 
      \frac{\bm{1}\tr\bm{\mu}_{\mathcal{B}} - \bm{1}\tr\bm{\mu}_{\mathcal{A} \cap \mathcal{B}}}{\bm{1}\tr\bm{\mu}_{\mathcal{A} \cup \mathcal{B}} - \bm{1}\tr\bm{\mu}_{\mathcal{A} \cap \mathcal{B}}} g(\bm{1}\tr\bm{\mu}_{\mathcal{A} \cup \mathcal{B}}).
  \end{align}
  Then summing~\eqref{eq:ga-bound} and~\eqref{eq:gb-bound} and algebraic manipulation yields that 
  \begin{equation} \label{eq:submod-proof-cond}
    f(\mathcal{A}) + f(\mathcal{B})
    = g(\bm{1}\tr\bm{\mu}_{\mathcal{A}}) + g(\bm{1}\tr\bm{\mu}_{\mathcal{B}})
    \ge
      g(\bm{1}\tr\bm{\mu}_{\mathcal{A} \cap \mathcal{B}}) + 
      g(\bm{1}\tr\bm{\mu}_{\mathcal{A} \cup \mathcal{B}})
     = 
      f(\mathcal{A} \cap \mathcal{B}) + 
      f(\mathcal{A} \cup \mathcal{B}),
\end{equation}
    which proves that $f$ is submodular. 
\end{proof}

\begin{proof}[Proof of \cref{prop:cvar-ews-correct}]
The function $f_{\alpha }^{\cvaro}$ is EWS by algebraic manipulation. The equality~\eqref{eq:cvar-correct} for $\alpha = 0$ follows by algebraic manipulation. To prove~\eqref{eq:cvar-correct} for $\alpha > 0$, we use \cref{prop:coherent-polymatroid} and show for each $\mathcal{O} \subseteq \Omega$ that
\begin{equation}
\label{eq:cvar-eq-submodular}
  f_{\alpha}^{\cvaro}(\mathcal{O})
  =
  \min\left\{ \alpha^{-1} \bm{1}_{\mathcal{O}}\tr \bm{p},\; 1 \right\}
= \max \left\{ \bm{1}_{\mathcal{O}}\tr \bm{q} \mid 
\bm{q} \in \Delta_{n}, \,
\bm{q} \le \alpha^{-1} \bm{p}
\right\}
=
  -\cvar{\alpha }{- \tilde{1}_{\mathcal{O}}},
\end{equation}
using the dual CVaR formulation in~\eqref{eq:primal_cvar}. We prove~\eqref{eq:cvar-eq-submodular} as two inequalities.

(i) The inequality 
\[
\min\left\{ \alpha^{-1} \bm{1}_{\mathcal{O}}\tr \bm{p},\; 1 \right\}
\ge \max \left\{ \bm{1}_{\mathcal{O}}\tr \bm{q} \mid \bm{q} \in \Delta_{n}, \, \bm{q} \le \alpha^{-1} \bm{p}
\right\},
\]
holds because $\bm{q}$ feasible in the right-hand side satisfies that 
\(  \bm{1}_{\mathcal{O}}\tr \bm{q} \le \alpha^{-1} \bm{1}_{\mathcal{O}}\tr \bm{p},
 \text{ and }
  \bm{1}_{\mathcal{O}}\tr \bm{q} \le \bm{1}\tr \bm{q} = 1.
\)

(ii) To prove that 
\begin{equation} \label{eq:cvar-proof-le}
\min\left\{ \alpha^{-1} \bm{1}_{\mathcal{O}}\tr \bm{p},\; 1 \right\}
\le \max \left\{ \bm{1}_{\mathcal{O}}\tr \bm{q} \mid \bm{q} \in \Delta_{n}, \, \bm{q} \le \alpha^{-1} \bm{p}
\right\},
\end{equation}
we analyze the following two cases. 
\begin{enumerate}[nosep]
\item[(a)] When $\bm{1}_{\mathcal{O}}\tr  \bm{p} \ge \alpha$, let \(\bm{q} :=  (\bm{1}_{\mathcal{O}}\tr \bm{p})^{-1} \cdot \bm{p} \odot \bm{1}_{\mathcal{O}},
\) which is
feasible in the right-hand side of~\eqref{eq:cvar-proof-le} and $\odot$ is the Hadamard vector product. 
\item[(b)] When $\bm{1}_{\mathcal{O}}\tr  \bm{p} < \alpha$, let \(\bm{q} := \alpha^{-1} \cdot \bm{p} \odot \bm{1}_{\mathcal{O}}
+  d \cdot \bm{p} \odot (\bm{1} - \bm{1}_{\mathcal{O}}), \) which is
feasible in the right-hand side of~\eqref{eq:cvar-proof-le} and $d\in \Real $ solves that
\(d \cdot \bm{1}\tr \bm{p} \odot (\bm{1} - \bm{1}_{\mathcal{O}}) = 1 - \alpha^{-1} \bm{1}_{\mathcal{O}}\tr \bm{p}.
\)
Such $d$ always exists because \(\bm{1}\tr \bm{p} \odot (\bm{1} - \bm{1}_{\mathcal{O}}) = 1 - \bm{1}_{\mathcal{O}}\tr \bm{p}  > 1 - \alpha \ge 0.
\)
It is easy to verify that $\bm{q} \ge \bm{0}$ and $\bm{1}\tr \bm{q} = 1$. The condition $\bm{q} \le \alpha^{-1} \bm{p}$ holds because $d \le \alpha^{-1}$ is satisfied by multiplying both sides of the inequality by $\bm{1}\tr \bm{p} \odot (\bm{1} - \bm{1}_{\mathcal{O}})$:
\[
  d \cdot \bm{1}\tr \bm{p} \odot (\bm{1} - \bm{1}_{\mathcal{O}})
  \; \le\; 
  1 - \frac{1}{\alpha } \bm{1}_{\mathcal{O}}\tr \bm{p}
  \; \le\; 
  \frac{1}{\alpha } \cdot \bm{1}\tr \bm{p} \odot (\bm{1} - \bm{1}_{\mathcal{O}}).
\]
\end{enumerate}
\end{proof}

\begin{proof}[Proof of \cref{prop:tvar-ews-correct}]
  The function $f_{\alpha }^{\tvaro}$ is EWS by algebraic manipulation. The equality~\eqref{eq:tvar-correct} for $\alpha$ such that $\sqrt{\frac{1}{2} \log \frac{1}{\alpha} } \ge 1$, including $\alpha = 0$, follows by algebraic manipulation.

  It remains to prove~\eqref{eq:tvar-correct} when $\sqrt{\frac{1}{2} \log \frac{1}{\alpha} } < 1$. We use \cref{prop:coherent-polymatroid} and show for each $\mathcal{O} \subseteq \Omega$ that
\begin{equation}\label{eq:tvar-eq-submodular}
  \begin{aligned}
f_{\alpha}^{\tvaro}(\mathcal{O}) 
  &=
\min\left\{ \bm{1}_{\mathcal{O}}\tr \bm{p} + \sqrt{\frac{1}{2} \log \frac{1}{\alpha}},\; 1 \right\} \\
  &=
\max \left\{ \bm{1}_{\mathcal{O}}\tr \bm{q} \mid 
\bm{q} \in \Delta_n, \,
\|\bm{p} -\bm{q}\|_1 \le \sqrt{2 \log \frac{1}{\alpha}} 
\right\}
=
  -\tvar{\alpha }{- \tilde{1}_{\mathcal{O}}},
  \end{aligned}
\end{equation}
using the TVaR definition in~\eqref{eq:tvar-dual}. We need to handle two special cases first. When $\bm{1}_{\mathcal{O}}\tr \bm{p} = 0$, then we have that $0 = f^{\tvaro}_{\alpha}(\mathcal{O}) = g(0) = -\tvaro_{\alpha}[-\tilde{1}_{\mathcal{O}}]$ since $\bm{q} \ll \bm{p}$ in the TVaR definition in~\eqref{eq:tvar-dual}. When $\bm{1}_{\mathcal{O}}\tr \bm{p} = 1$, then we have that $1 = f^{\tvaro}_{\alpha}(\mathcal{O}) = g(1) = -\tvaro_{\alpha}[-\tilde{1}_{\mathcal{O}}]$ since $\bm{q} \ll \bm{p}$ and $\bm{1}\tr \bm{q} = 1$ in the TVaR definition in~\eqref{eq:tvar-dual}.

Finally, we prove~\eqref{eq:tvar-eq-submodular} as two inequalities when $0 < \bm{1}_{\mathcal{O}}\tr \bm{p} < 1$. 
(i) To prove that 
\[
  \min\left\{ \bm{1}_{\mathcal{O}}\tr \bm{p} + \sqrt{\frac{1}{2} \log \frac{1}{\alpha}},\; 1 \right\}
  \ge \max \left\{ \bm{1}_{\mathcal{O}}\tr \bm{q} \mid
\bm{q} \in \Delta_n, \,
\|\bm{p} -\bm{q}\|_1 \le \sqrt{2 \log \frac{1}{\alpha}}
\right\},
\]
observe that $\bm{q}$ feasible in the right-hand side implies that \(  \|\bm{p} - \bm{q}\|_1 = \bm{1}_{\mathcal{O}}\tr |\bm{p} - \bm{q}| + (\bm{1} - \bm{1}_{\mathcal{O}})\tr |\bm{p} - \bm{q}|
  \leq \sqrt{2 \log \frac{1}{\alpha }} \) and $(\bm{1} - \bm{1}_{\mathcal{O}})\tr (\bm{p} - \bm{q}) = \bm{1}_{\mathcal{O}}\tr  (\bm{q} - \bm{p})$ since $\bm{p},\bm{q}\in \Delta_n$.  Then
\begin{align*}
  \bm{1}_{\mathcal{O}}\tr \bm{q}
  &= \bm{1}_{\mathcal{O}}\tr (\bm{q} - \bm{p}) + \bm{1}_{\mathcal{O}}\tr  \bm{p} 
  = \frac{1}{2}\left(  \bm{1}_{\mathcal{O}}\tr (\bm{q} - \bm{p})  + (\bm{1} - \bm{1}_{\mathcal{O}})\tr  (\bm{p} - \bm{q}) \right)+ \bm{1}_{\mathcal{O}}\tr  \bm{p} \\
  &\le \frac{1}{2}\left(  \bm{1}_{\mathcal{O}}\tr |\bm{q} - \bm{p}|  + (\bm{1} - \bm{1}_{\mathcal{O}})\tr |\bm{p} - \bm{q}| \right)+ \bm{1}_{\mathcal{O}}\tr  \bm{p} 
  = \sqrt{\frac{1}{2} \log \frac{1}{\alpha }} + \bm{1}_{\mathcal{O}}\tr  \bm{p},
\end{align*}
and  \(  \bm{1}_{\mathcal{O}}\tr \bm{q} \leq \bm{1}\tr \bm{q}  = 1.\)

(ii) To prove that 
\[
  \min\left\{ \bm{1}_{\mathcal{O}}\tr \bm{p} + \sqrt{\frac{1}{2} \log \frac{1}{\alpha }},\; 1 \right\}
\le \max \left\{ \bm{1}_{\mathcal{O}}\tr \bm{q} \mid 
\bm{q} \in \Delta_n, \,
\|\bm{p} -\bm{q}\|_{1} \le  \sqrt{2 \log \frac{1}{\alpha }}
\right\},
\]
we analyze the following two cases:
\begin{enumerate}[nosep]
\item When $\bm{1}_{\mathcal{O}}\tr  \bm{p} +  \sqrt{\frac{1}{2} \log \frac{1}{\alpha }} \ge 1$, let \(\bm{q} \;:=\;  (\bm{1}_{\mathcal{O}}\tr \bm{p})^{-1} \cdot \bm{p} \odot \bm{1}_{\mathcal{O}}, \)
  where $\odot$ is the Hadamard vector product. This $\bm{q}$ is feasible in the right-hand side and satisfies the inequality. 
\item When $\bm{1}_{\mathcal{O}}\tr  \bm{p}  + \sqrt{\frac{1}{2} \log \frac{1}{\alpha }}< 1$, let 
\[
\bm{q} \;:=\; \frac{\bm{1}_{\mathcal{O}}\tr \bm{p} + \sqrt{\frac{1}{2} \log \frac{1}{\alpha }}}{\bm{1}_{\mathcal{O}}\tr \bm{p}} \cdot \bm{p} \odot \bm{1}_{\mathcal{O}}
+  \frac{(\bm{1} - \bm{1}_{\mathcal{O}}) \tr \bm{p} - \sqrt{\frac{1}{2} \log \frac{1}{\alpha }} }{(\bm{1} - \bm{1}_{\mathcal{O}}) \tr \bm{p}}\cdot \bm{p} \odot (\bm{1} - \bm{1}_{\mathcal{O}}),
\]
which is feasible in the right-hand side and satisfies the inequality. Neither one of the divisions is by $0$ since $0 < \bm{1}_{\mathcal{O}}\tr \bm{p} < 1$.
\end{enumerate}
\end{proof}

\end{document}